\documentclass[runningheads]{llncs}
\usepackage[T1]{fontenc}

\usepackage{graphicx}
\usepackage{amsmath}
\usepackage{amssymb}
\usepackage[dvipsnames]{xcolor}
\usepackage{subcaption}
\graphicspath{ {images/} }
\usepackage{float}
\usepackage{listings}
\usepackage{algorithm,algcompatible}
\newcommand{\norm}[1]{\left\lVert#1\right\rVert}
\usepackage{algorithmicx}
\usepackage{algpseudocode}
\usepackage{makecell}
\usepackage{booktabs}       
\usepackage{url}
\usepackage{multirow}
\usepackage{bbm}
\usepackage{soul}
\usepackage{hyperref}

\newcommand{\R}{\mathbb{R}}
\newcommand{\Z}{\mathbb{Z}}
\newcommand{\B}{\mathbb{B}}

\newcommand{\uu}{\mathbf{u}}

\newcommand{\s}{\mathbf{s}}

\usepackage{CJKutf8}

\begin{document}
\title{Repeatable and Reliable Efforts of Accelerated Risk Assessment in Robot Testing}
\titlerunning{Repeatable and Reliable Efforts of Accelerated Risk Assessment}
%
\author{Linda Capito\inst{1} \and
Guillermo A. Castillo\inst{2} \and
Bowen Weng\inst{3}}
\authorrunning{Capito et al.}
%
\institute{Transportation Research Center Inc., East Liberty OH 43319, USA \email{capitol@trcpg.com}\and
The Ohio State University, Columbus OH 43210 USA \email{castillomartinez.2@osu.edu}\and
Iowa State University, Ames IA 50011, USA
\email{bweng@iastate.edu}}

\maketitle              
\begin{abstract}
Risk assessment of a robot in controlled environments, such as laboratories and proving grounds, is a common means to assess, certify, validate, verify, and characterize the robots' safety performance before, during, and even after their commercialization in the real-world. A standard testing program that acquires the risk estimate is expected to be (i)~repeatable, such that it obtains similar risk assessments of the same testing subject among multiple trials or attempts with the similar testing effort by different stakeholders, and (ii) reliable against a variety of testing subjects produced by different vendors and manufacturers. Both repeatability and reliability are fundamental and crucial for a testing algorithm's validity, fairness, and practical feasibility, especially for standardization. However, these properties are rarely satisfied or ensured, especially as the subject robots become more complex, uncertain, and varied. This issue was present in traditional risk assessments through Monte-Carlo sampling, and remains a bottleneck for the recent \emph{accelerated} risk assessment methods, primarily those using importance sampling. This study aims to enhance existing accelerated testing frameworks by proposing a new algorithm that provably integrates repeatability and reliability with the already established formality and efficiency. It also features demonstrations assessing the risk of instability from frontal impacts, initiated by push-over disturbances on a controlled inverted pendulum and a 7-DoF planar bipedal robot Rabbit managed by various control algorithms.

\end{abstract}
\section{Introduction}



The \emph{repeatability}~\cite{mooring1986determination,falco2020benchmarking}, also known as reproducibility~\cite{pineau2021improving,impagliazzo2022reproducibility} and replicability~\cite{eaton2024replicable}, is a commonly observed term across a variety of disciplines and applications. Repeatability of measurements refers to ``the variation in repeat measurements made on the same subject under identical conditions''~\cite{taylor1994guidelines}.
In the context of testing, such a measurement can be very complex, involving not only the metric applied to the intermediate testing results for performance interpretation, but also the strategies of exploring states and actions to obtain the initial results. This makes the measurement an \emph{algorithm}, or a \emph{testing algorithm}. Within the context of standard testing, repeatability is thus characterized as the capability to achieve similar testing outcomes across multiple trials or attempts of the same \emph{testing algorithm} against the same subject.

\emph{Reliability}, on the other hand, is the characterization of the testing algorithm's consistent utility across a variety of testing subjects. It partially relies on the algorithm being repeatable in the first place, but it embodies other properties, such as the expectation for the testing efforts to remain bounded and fairness concerns among a diverse set of subject systems being tested.

Note both terms---repeatability and reliability---have been used in different scientific disciplines and research fields, which may not share the same concerns and meanings. This work primarily focuses on the standard testing of robots; hence, these terminologies are specifically adopted and formally defined for this context in Section~\ref{sec:prob}.


From the perspective of standardized testing, which is highly valued by third-party entities such as regulatory and standard organizations, the repeatability and reliability of a testing algorithm is crucial to ensuring the validity, fairness, and feasibility of a testing program. Moreover, these properties must be established as \emph{a priori} knowledge at the design phase before the execution of any testing algorithm. Consequently, they are expected to be \emph{provable} properties rather than empirical demonstrations.

\subsection{Accelerated testing for risk estimate}
In this paper (and many standard testing literature~\cite{zhao2016accelerated,feng2023dense}), ``risk'' is considered as the probability of the occurrence of a specific set of undesirable events (e.g., collision and falling over) subject to a given target distribution.
At a high level, a testing algorithm aimed at risk estimation operates as a ``statistical query'' algorithm~\cite{feldman2017general} that explores states and actions in a Monte-Carlo manner to experimentally determine the failure probability of the testing subject. This approach is frequently seen as cumbersome because it demands a substantial number of samples~\cite{liu2022curse}. This requirement becomes particularly pronounced in safety performance testing, where instances of failure tend to happen with low probability, making it a challenging and sample-intensive process.

Motivated by this difficulty, the concept of importance sampling, along with its various adaptations, has shown promising performance for safety testing of Automated Driving Systems (ADS)~\cite{ding2020learning,ding2021multimodal}, Advanced Driver Assistant Systems (ADAS)~\cite{zhao2016accelerated,zhao2017accelerated}, general software~\cite{gutjahr1997importance}, to name a few.

\begin{figure*}[t]
    \centering
    \begin{subfigure}{0.5\textwidth}
        \includegraphics[width=\textwidth]{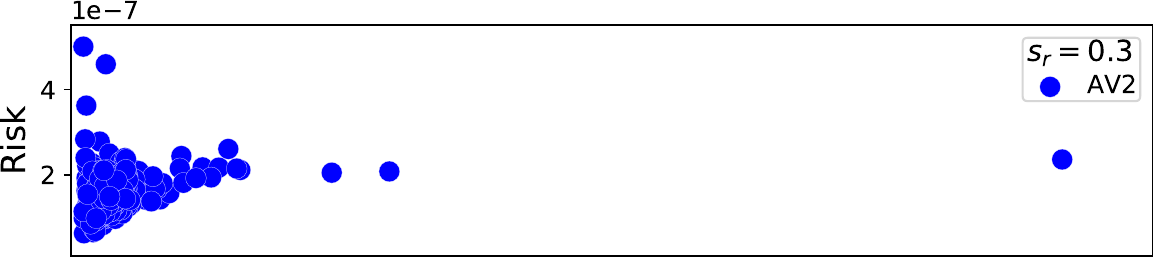}
        \caption{}
        \label{fig:av2_0_3}        
    \end{subfigure}
    \hfill 
    \begin{subfigure}{0.48\textwidth}
        \includegraphics[width=\textwidth]{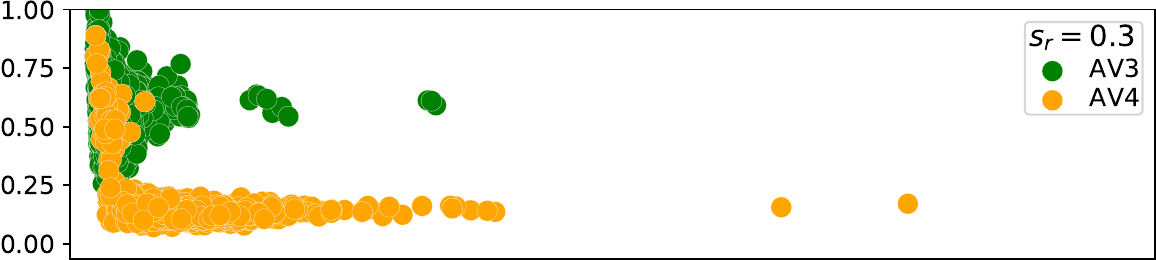}
        \caption{}
        \label{fig:random_0_3}        
    \end{subfigure}
    \newline 
    \begin{subfigure}{0.5\textwidth}
        \includegraphics[width=\textwidth]{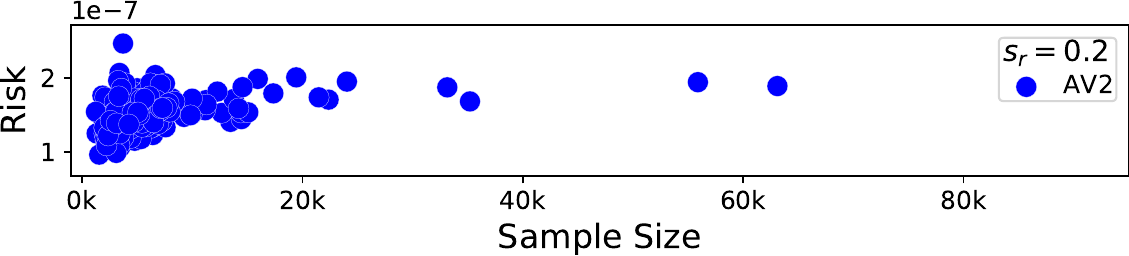}
        \caption{}
        \label{fig:av2_0_2}        
    \end{subfigure}
    \hfill 
    \begin{subfigure}{0.49\textwidth}
        \includegraphics[width=\textwidth]{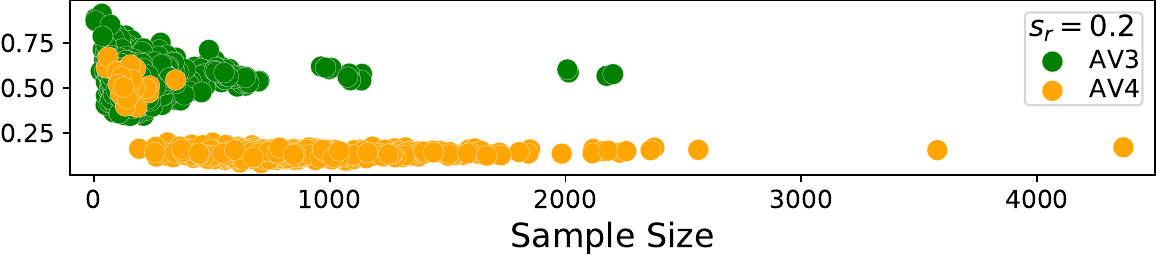}
        \caption{}
        \label{fig:random_0_2}        
    \end{subfigure}  
    \vspace{-3mm}
    \caption{The repeatability and reliability issues of NADE, an importance sampling inspired ADS risk assessment algorithm~\cite{feng2021intelligent}, revealed through an extended re-implementation of the open-source code~\cite{nadegithub}. Overall, three different ADS algorithms are tested, including AV2 (a default subject from the original proposal~\cite{feng2021intelligent}) and two customized algorithms referred to as AV3 and AV4. The hyper-parameter $s_r$ is a threshold value for the relative half-width of the risk estimate, which is closely related to a widely adopted empirical termination condition for importance-sampling based testing algorithms (see Section~\ref{sec:prob} for more details).}
    \label{fig:is_example}
    \vspace{-5 mm}
\end{figure*}

However, for a sampling based testing algorithm---whether oriented towards importance sampling or adhering to a nominal distribution---the challenges of ensuring repeatability and reliability become amplified. This phenomenon is demonstrated in Fig.~\ref{fig:is_example}, where we extend the original implementation~\cite{nadegithub} of an importance sampling inspired ADS testing proposal~\cite{feng2021intelligent} to repeated trials against different subject driving algorithms. The proposal in~\cite{feng2021intelligent} is referred to as Naturalistic and Adversarial Driving Environment (NADE). In Fig.~\ref{fig:av2_0_3}, 100 repeated trials of the default implementation of NADE lead to 100 different risk estimates of AV2, the subject vehicle being tested (original implementation from~\cite{feng2021intelligent}), with significantly different testing efforts highlighted by the range of sample sizes. This lack of repeatability is not confined to a single set of hyper-parameters (e.g., see Fig.~\ref{fig:av2_0_2}) or to AV2 alone. Similar inconsistencies in both testing effort and risk estimation were observed when testing additional vehicle algorithms, referred to as AV3 and AV4 as depicted in Fig.~\ref{fig:random_0_3} and Fig.~\ref{fig:random_0_2}. Note both AV3 and AV4 are customized policies following different behavioral probability distributions of selecting driving actions. They are also dramatically different from (generally worse than) AV2.

Moreover, as the testing subject changes from AV2 to AV3 and AV4, the overall testing effort significantly changes, with the maximum sample size decreasing from nearly 90,000 tests to about 4,000. This means testing different subjects could result in significantly different testing efforts, and the difference is not foreseeable before the execution of the tests following the traditional approach. Such discrepancies, evident in the referenced figures, signal a significant reliability challenge for safety testing algorithms.


\begin{remark}
    This study focuses on importance sampling based testing algorithms. While other testing algorithms for falsification~\cite{lee2019adaptive,capito2020modeled}, verification~\cite{fan2017d,kwiatkowska2007stochastic}, and validation~\cite{koren2019efficient,weng2021towards} may share similar concerns of repeatability and reliability, they are out of the scope of this paper.
\end{remark}

\subsection{Main Contributions}
Inspired by the aforementioned issues, this work seeks to present, to the best of the authors' knowledge, the first formal study on repeatable and reliable accelerated risk assessment algorithms using importance sampling techniques.

Motivated by works on reproducible learning algorithms~\cite{impagliazzo2022reproducibility} and sample size analyses of importance sampling~\cite{chatterjee2018sample}, we propose a risk assessment algorithm (Algorithm~\ref{alg:r2_tight} in Section~\ref{sec:main}) that (i) provides an unbiased risk estimate with high accuracy and high probability, contingent upon a predefined and finite number of samples across a set of testing subjects (reliability), and (ii) obtains the same risk estimate with high probability across multiple testing attempts (repeatability). These properties are substantiated under certain theoretically sound and practically attainable conditions.

It is also important to note that the procedure described in Section~\ref{sec:prob:direct} is repeatable and reliable, and to the best of our knowledge, has not yet been explored in the \emph{testing literature}. However, it directly extends the repeatable statistical query algorithm discussed in~\cite{impagliazzo2022reproducibility}, hence is not considered as a main proposal. More importantly, it has limitations that fail to take advantage of the distribution dependent nature of importance sampling, which will be addressed by the main proposal (i.e., Algorithm~\ref{alg:r2_tight}) presented in Section~\ref{sec:main}. The empirical comparison between the main proposal and this direct extension from~\cite{impagliazzo2022reproducibility} is also presented in Fig.~\ref{fig:effort_pendulum} (Section~\ref{sec:exp}).

The proposed algorithms are initially showcased through a pendulum push-over task, which facilitates ablation studies due to the availability of ground-truth outcomes. This example also directly relates to real-world scenarios, such as push-over disturbance rejection tests in legged robots, as described in~\cite{weng2022safety,weng2022rethink,weng2023towards}. This is illustrated in the second case study, which examines the risk of falling over in a 7-DoF (Degree-of-Freedom) planar bipedal robot Rabbit controlled by model-based~\cite{gong2021one} and learning-oriented locomotion approaches~\cite{castillo2019reinforcement}.

\section{Preliminaries and Problem Formulation}\label{sec:prob}
\textbf{Notation: } The set of real and positive real numbers are denoted by $\R$ and $\R_{>0}$ respectively. $\Z$ denotes the set of all positive integers and $\Z_N=\{1,\ldots,N\}$. $\B$ denotes the set \{0,1\}. The $\ell_{2}$-norm is denoted by $\norm{\cdot}$. If $X$ is a set, $|X|$ is its cardinality. If $x$ is a scalar in $\R$, $|x|=\norm{x}_1$ denotes its absolute value. The sign function is defined as $\text{sgn}(x)=1 \text{ if }x\geq0 \text{ else } 0$. The Kronecker delta function is defined as $\delta_{i,j}=1 \text{ if }i=j \text{ else } \delta_{i,j}=0$.

\subsection{The testing system}
The notion of testing fundamentally admits a control system formulation as 
\begin{equation}\label{eq:ctrl-sys}
    \s(t+1) = f(\s(t), \uu(t); \omega(t)).
\end{equation}
The state $\s \in S \subset \R^n$ includes the characterization of all observable properties of all testing participants and environmental conditions, such as a vehicle's position, velocity, and the weather condition. The testing action $\uu \in U \subset \R^m$ denotes the controllable factors of the testing system. For the inverted pendulum push-over tests (Section~\ref{sec:exp}), $\uu$ could denote the contact point, magnitude and direction of the push-over disturbance. For vehicle collision avoidance tests (e.g., Fig.~\ref{fig:is_example}), $\uu$ typically represents the behaviors of other traffic participants such as a lead vehicle or a crossing pedestrian. Finally, $\omega \in W$ represents the disturbances and uncertainties (e.g., wind speed and road friction). Note $S^{S\times U \times W}$ represents the set of all testing systems (including different testing subjects, such as two vehicles with different makes and models or the same vehicle with different software versions).

The execution of a test always starts from a certain initialization state $\s(0) \in S$, and propagates with a certain scheme that determines the testing action $\uu(t)$. Furthermore, the testing action can be propagated via two distinct methods, namely the open-loop and the feedback approaches. The open-loop approach presents a finite sequence of $\xi$ actions $\bar{\uu}=\{\uu(t)\}_{t\in\Z_{\xi}}$, whereas the feedback approach determines the testing action based on instantaneous state feedback as $\uu(t)=\pi(\s(t))$, where $\pi$ is the policy mapping states to actions. 

\subsection{Sampling-based risk estimate}
A testing algorithm for risk estimation $\mathcal{TE}$ establishes a strategy that explores samples of testing initialization states $\s_0$ and testing actions $\bar{\uu}$, collects trajectories of states $T \subset S^{\xi}$ per the testing system dynamics \eqref{eq:ctrl-sys}, analyzes the risky outcome (e.g., collision or non-collision) with $\mathcal{C} \subset \R^n$ representing the set of unsafe states, and terminates based on a specific criterion $\mathcal{T}: S^{\Z \times \xi} \rightarrow \B$ (that takes all states collected from all tests as the input). In this paper (and the general practice of risk assessment testing algorithms), this termination condition is triggered either by reaching a predefined number of collected trajectories or by other indicators of the convergence of the risk estimate.

In the open-loop testing configuration, the state initialization follows a certain probability distribution $p_s$ over the sample space $S$. Similarly, the open-loop actions are determined by the distribution $p_u$ over the action space $U^{\xi}$. Since the samples of states and actions are independent, one can denote $p(x) = p(\s_0, \bar{\uu}) = p_s(\s_0)\cdot p_u(\bar{\uu})$ for some $\s_0$ and $\bar{\uu}$. Consequently, the feedback testing policy represents a particular case of this general characterization, integrating the testing system dynamics as described in~\eqref{eq:ctrl-sys}, leading to $p(x) = p_s(\s_0)$. The above steps are formally described in Algorithm~\ref{alg:testing}.

\begin{algorithm}
    \begin{algorithmic}[1]
    \State {\bf Given: $X$, $p$, $\xi$, $\mathcal{T}$, $\mathcal{C}, f$}
    \State {\bf Initialize: $i=1, T=\emptyset$ }
    \State {{\bf While} $\mathcal{T}(T)=0$:}
    \State {\ \ \ \ $(\s_0, \bar{\uu})_i = x_i \sim p(X)$}
    \State {\ \ \ \ Collect state trajectory $T_i$ through $f$~\eqref{eq:ctrl-sys}}
    \State {\ \ \ \ $r = \Big(\sum_{j=1}^i(1-\delta_{T_i\cap \mathcal{C}, \emptyset})\cdot p_s(x_j)\Big)/i$}
    \State {\ \ \ \ $T.\texttt{append}(T_i)$, $i$+=$1$}
    \State {{\bf Output}: $r_i$}
    \end{algorithmic}
    \caption{$\mathcal{TE}(X, p, \xi, \mathcal{T}, \mathcal{C}, f)$} \label{alg:testing}
\end{algorithm}

The distribution $p$ over the sample space $X = S \times U^{\xi}$ is commonly referred to as the \emph{nominal} distribution (or target distribution). This distribution represents the probability of disturbances, uncertainties, and controllable changes within the context for which the risk estimate is sought. In the vehicle testing context, it is often characterized as a naturalistic driving model learned from large-scale real-world human driver data~\cite{feng2021intelligent,feng2023dense}.

\begin{algorithm}
    \begin{algorithmic}[1]
    \State {\bf Given: $X$, $p$, $q$, $\xi$, $\mathcal{T}$, $\mathcal{C}, f$}
    \State {\bf Initialize: $i=1, T=\emptyset$ }
    \State {{\bf While} $\mathcal{T}(T)=0$:}
    \State {\ \ \ \ $(\s_0, \bar{\uu})_i = x_i \sim q(X)$}
    \State {\ \ \ \ Collect state trajectory $T_i$ through $f$~\eqref{eq:ctrl-sys}}
    \State {\ \ \ \ $r = \Big(\sum_{j=1}^i(1-\delta_{T_j\cap \mathcal{C}, \emptyset})\cdot p(x_j)/q(x_j) \Big)/i$}
    \State {\ \ \ \ $T.\texttt{append}(T_i)$, $i$+=$1$}
    \State {{\bf Output}: $r_i$}
    \end{algorithmic}
    \caption{$\mathcal{TE}_{IS}(X, p, q, \xi, \mathcal{T}, \mathcal{C}, f)$} \label{alg:is_testing}
\end{algorithm}

Note that achieving a risk estimate via Algorithm~\ref{alg:testing} in a Monte-Carlo manner often requires a significant amount of tests. This demand is largely due to the low probability of having $T_i\cap \mathcal{C} \neq \emptyset$ (i.e., a failure test), a phenomenon variously described as long-tail event, corner case, rare event, and curse-of-rarity~\cite{liu2022curse}. To handle these low-probability events, many testing algorithms employ importance-sampling, adopting a structure similar to that presented in Algorithm~\ref{alg:is_testing}.

The distribution $q$ is referred to as the \emph{importance} distribution, specifically designed to bias the exploration towards state-action pairs associated with higher risk levels. Considering that the underlying ground-truth collision rate is $r^*$, it has been shown theoretically and practically that both algorithms can provide an unbiased estimate of $r^*$. That is, suppose the termination condition requires a sufficiently large number of $n$ testing trajectories to be collected, i.e. with $\mathcal{T}(T) = \delta_{\text{sgn}(|T|-n),1}$, $\lim_{n \rightarrow \infty} |\mathcal{TE}(\cdot)-r^*| = 0$ and $\lim_{n \rightarrow \infty} |\mathcal{TE}_{IS}(\cdot)-r^*| = 0$.

In theory, the minimum number of tests required for the above two algorithms depends on a few given parameters. Among those that are generally configurable (e.g., confidence level and desired accuracy), the variance of the risk estimates fundamentally requires the analytical knowledge of $f$, which is not attainable in real-world applications~\cite{wasserman2013all,ross2017introductory}. In practice, the termination condition $\mathcal{T}$ is commonly implemented as a ``runtime'' empirical observer of the relative half-width (RHW) of the estimations w.r.t. a predetermined threshold $s_r \in \R _{>0}$ supplied with a certain confidence interval~\cite{zhao2017accelerated,feng2020testing,feng2021intelligent,feng2023dense}. For a sufficiently large $n$ or for a sufficiently small RHW threshold $s_r$, both algorithms are expected to ``converge'' to the $r^*$ asymptotically. Note \emph{this only guarantees the probabilistic accuracy of the testing algorithm rather than the repeatability of the testing algorithm}. That is, one has an estimate that is statistically close to $r^*$, but one does not have the \emph{same} estimate with high probability among multiple attempts of the same algorithm with the same $f$ and other hyper-parameters. As a result, they are not $\beta$-repeatable per Definition~\ref{def:beta-r} to be presented in the next section.

Moreover, the variance-based criterion and RHW form only a necessary condition for good performance. The threshold $s_r$ is typically fixed for all testing attempts and testing subjects. When RHW is used as the basis of a runtime termination observer, one could end up with multiple different risk estimates, all satisfying the given threshold (i.e. the repeatability problem). More importantly, these different estimates could come from substantially different testing efforts. The number of tests required could also fluctuate widely as one tests different functions $f \in S^{S\times U \times W}$. This variability introduces further complications, as there is no clear way to anticipate the extent of testing efforts or their variability across different test subjects. As a result, the previously mentioned two algorithms fail to meet the criteria to be $\gamma$-reliable per Definition~\ref{def:gamma-r} to be presented in the next section.




\subsection{Repeatable \& reliable risk estimate}
To formally address the repeatability and reliability issues of sampling-based risk assessment testing algorithms, the following definitions of $\beta$-Repeatability and $\gamma$-Reliability are presented with respect to the conceptual discussions mentioned above. Definition~\ref{def:beta-r} is adapted from the literature of reproducible learning algorithms~\cite{impagliazzo2022reproducibility}.
\begin{definition}\label{def:beta-r} [\textbf{$\beta$-Repeatability}]
    A sampling-based risk-estimate testing algorithm $\mathcal{TE}$ (e.g. Algorithm~\ref{alg:testing} and Algorithm~\ref{alg:is_testing}) is $\beta$-repeatable if for some $\beta \in (0,1)$ and any two executions of the same algorithm with all other configurations following the same level of randomness, $\mathcal{TE}^i$, $\mathcal{TE}^j$ and $i\neq j$,    
    \begin{equation}
        \mathbb{P}(\mathcal{TE}^i(\cdot) = \mathcal{TE}^j(\cdot)) \geq 1-\beta.  
    \end{equation}
\end{definition}
While a $\beta$-repeatable algorithm is often defined for a fixed testing subject revealed through the specifically given $f$, the reliability property extends the study to broader range of subject systems.
\begin{definition}\label{def:gamma-r}[\textbf{$\gamma$-Reliability}]
    Let $T_{jk}$ denote the set of state trajectories collected from the $j$-th execution of a testing algorithm, $\mathcal{TE}^j$, with a certain testing system $f_k \in S^{S \times U \times W}$. A sampling-based risk-estimate testing algorithm $\mathcal{TE}$ (e.g. Algorithm~\ref{alg:testing} and Algorithm~\ref{alg:is_testing}) is $\gamma$-reliable if, with all other configurations following the same level of randomness, it is $\beta$-repeatable for some $\beta \in (0,1)$, and there exists $\gamma \in \Z$, $\mathcal{F} \subset S^{S \times U \times W}$, and
    \begin{equation}
        \forall f_k \in \mathcal{F}, \forall j \in \Z, |T_{jk}| \leq \gamma
    \end{equation}
\end{definition}
Note that the $\beta$-repeatability is a prerequisite of the $\gamma$-reliability. However, while $\beta$-repeatability mostly defines a lower bound on the expected testing efforts to guarantee the desired properties (as will be shown in Theorem~\ref{thm:r2_tight}), $\gamma$-reliability conversely sets an upper bound on the testing efforts for all possible testing subjects. 


\subsection{A direct adaptation of reproducible statistical query}\label{sec:prob:direct}
This section concludes with a direct adaptation of the reproducible statistical query~\cite{impagliazzo2022reproducibility} to make a $\beta$-repeatable and $\gamma$-reliable risk assessment algorithm. Though, to the best of our knowledge, the specific implementation of this approach has not been found in the literature of testing, the adaptation is straightforward from ``reproducible learning''~\cite{impagliazzo2022reproducibility}, and the principles revealed by this adaptation also support the main proposal, as discussed in Section~\ref{sec:main}. 

Consider a risk estimation algorithm (e.g., Algorithm~\ref{alg:testing} and Algorithm~\ref{alg:is_testing}) with the estimated risk $r \in [0,1]$ as the output. The key idea inspired by~\cite{impagliazzo2022reproducibility} is to \emph{split} the output set of risk estimate, $[0,1]$, into multiple intervals of $\alpha$ resolution. Instead of directly returning the original risk estimate as the output, the algorithm returns the median of each interval that contains the original estimate. Naturally, the selection of $\alpha$ significantly influences the output. To mitigate this sensitivity, the proposed algorithm statistically determines the resolution $\alpha_0$ for the first interval, $[0, \alpha_0]$, by uniformly sampling $\alpha_0$ from the range $[0, \alpha]$. The rest of the intervals are unchanged. This modification lays the groundwork for a probabilistically ensured repeatability property.

Specifically, applying the above procedure to the output of Algorithm~\ref{alg:testing} or Algorithm~\ref{alg:is_testing}, with probability $1-\beta$, it produces consistent risk estimates (repeatability), and with probability $1-\epsilon$ ($\epsilon \in (0,1)$), that estimation error will be confined within $\tau > 0$ in the $\ell_1$-norm. Moreover, the above properties are sufficiently guaranteed with a total of $\frac{4 \log{(2/\epsilon)}}{2\tau^2(\beta-2\epsilon)^2}$ (by Theorem 2.3 in~\cite{impagliazzo2022reproducibility}) sampled tests (reliability). 

However, note the properties are ensured \emph{without any dependencies upon the sampling distributions $p$ and $q$}. This independence from the sampling distributions is a fundamental property embedded in the original reproducible learning proposal~\cite{impagliazzo2022reproducibility}. As a result, different sampling distributions do not affect performance outcomes. This insight serves as the foundation for the development of Algorithm~\ref{alg:r2_tight} in the following section.

\section{Main Proposal}\label{sec:main}
This section presents a provably repeatable and reliable risk assessment testing algorithms specifically for the class of importance-sampling based accelerated risk assessment algorithms. The \emph{repeatability} property draws techniques revealed in Section~\ref{sec:prob:direct} adapted from~\cite{impagliazzo2022reproducibility}. The \emph{reliability} property leverages the analysis of sample sizes in importance sampling~\cite{chatterjee2018sample}. Combining them together, one overlays the desired repeatability and reliability properties on the already achieved accuracy and efficiency of a typical importance-sampling based testing algorithm.

\begin{algorithm}
    \begin{algorithmic}[1]
    \State {\bf Given: $X$, $\xi$, $\mathcal{C}, f, \beta, \tau, c_r$, $p$, $q$, $\bar{r}$}
    \State {$\alpha= \frac{2\tau}{\beta+1}$}
    \State {$\alpha_0 \sim U([0, \alpha])$}
    \State {c=0}
    \State {{\bf While} $ \bar{r}  e^{-\frac{c}{4}} + 2\bar{r}\sqrt{\mathbb{P}\Big(\log \frac{p(x)}{q(x)} > D(p \mid\mid q) + \frac{c}{2}\Big)}  > \frac{\beta\tau}{\beta+1}$:}
    \State {\ \ \ \ $c=c+c_r$}
    \State {Let $n=e^{D(p \mid\mid q)+c}$ and $\mathcal{T}(T) = \delta_{\text{sgn}(|T|-n),1}$}
    \State {$r = \mathcal{TE}_{IS}(X, p, q, \xi, \mathcal{T}, \mathcal{C}, f)$}
    \State {{\bf If} $r \leq \alpha_0$}
    \State {\ \ \ \ {\bf Output}: $\frac{\alpha_0}{2}$}
    \State {{\bf Else}:}
    \State {\ \ \ \ {\bf Output}: $\frac{\alpha}{2}+ \alpha \cdot \texttt{round}((r-\alpha_0)/\alpha)$}
    \end{algorithmic}
    \caption{$\overline{\mathcal{TE}}_{IS}$($X$, $\xi$, $\mathcal{C}, f, \beta, \tau, c_r$, $p$, $q$, $\bar{r}$)} \label{alg:r2_tight}
\end{algorithm}

\subsection{The main algorithm, theorem, and proof}
The main algorithm is given as Algorithm~\ref{alg:r2_tight}, which holds the repeatability and reliability property revealed through Theorem~\ref{thm:r2_tight}.

\begin{theorem}\label{thm:r2_tight}
    Given $\bar{\tau}\in (0,1]$, Algorithm~\ref{alg:r2_tight} is $\beta$-repeatable for $\beta \in (0,1)$ satisfying $|\widehat{\mathcal{TE}}_{IS}(\cdot)-r^*|\leq \tau\in \R_{>0}$ except with probability $2e^{-2n\tau^2}, n=e^{D(p \mid\mid q)+c}$ for some $c\in\R_{\geq0}$. 
    Moreover, for all importance distribution $q$ satisfying
    \begin{equation}\label{eq:r2_tight_ieq}
        \bar{r} \Bigg(e^{-\frac{c}{4}} + 2\sqrt{\mathbb{P}\Big(\log \frac{p(x)}{q(x)} > D(p \mid\mid q) + \frac{c}{2}\Big)} \Bigg) \leq \frac{\beta\tau}{\beta+1},
    \end{equation}
    and some $c \in [0, -D(p \mid\mid q) + \log\gamma]$, it is also $\gamma$-reliable.
\end{theorem}
\begin{proof}
    Let $\phi_f: X \rightarrow \B$ be the failure check function (which composes the testing system dynamics \eqref{eq:ctrl-sys}, state trajectory collection, and the intersection check against $\mathcal{C}$) with a boolean output (w.l.o.g., let $\phi_f(x) \in \{0,1\}, \forall x$). We have $r^* = \mathbb{E}_p \phi(x)$ and the risk estimate through Algorithm~\ref{alg:is_testing} as $r=\frac{1}{n}\sum_{i=1}^n \phi_f(x_i)p(x_i)/q(x_i)$. Let $L = D(p \mid\mid q)$. Then, for $n=e^{L+t}$ with some $t \in \R_{\geq0}$,
    \begin{equation}\label{eq:r2_tight_error_bound}
        \begin{aligned}
            & \mathbb{E}|r - r^*| \\
            & \leq \sqrt{\mathbb{E}(\phi_f(x)^2)} \Bigg(e^{-\frac{c}{4}} + 2\sqrt{\mathbb{P}\Big(\log \frac{p(x)}{q(x)} > L + \frac{c}{2}\Big)} \Bigg) \\
            & \leq \bar{r} \Bigg(e^{-\frac{c}{4}} + 2\sqrt{\mathbb{P}\Big(\log \frac{p(x)}{q(x)} > L + \frac{c}{2}\Big)} \Bigg) \leq \tau'.
        \end{aligned} 
    \end{equation}
    The first inequality is due to Theorem 1.1 in \cite{chatterjee2018sample}. The second inequality is derived as $r \in [0,1]$ and $\phi_f(x) \in \{0,1\}$, hence $\mathbb{E}(\phi_f(x)^2) = \mathbb{E}(\phi_f(x)) \leq \bar{r} = \max_{f \in \mathcal{F}} \sum_{x\in X} \phi_f(x)p(x).$
    
    Let $\tau' = \frac{\beta\tau}{\beta+1}$ and $\alpha= \frac{2\tau}{\beta+1}$. Algorithm~\ref{alg:r2_tight} offsets Algorithm~\ref{alg:is_testing}'s output by at most $\alpha/2$, as a result, $\mathbb{E}|r - r^*| \leq \tau' + \frac{\alpha}{2} = \frac{\tau+\beta\tau}{\beta+1} = \tau.$
    
    With $|r - r^*|\leq 1$, and by Hoeffding's inequality, we further have $\mathbb{P}(|r - r^*| \geq \tau) \leq 2e^{-2n\tau^2}$. This proves the $\beta$-repeatability of estimation accuracy $\tau$ except with probability $2e^{-2n\tau^2}$.
    Furthermore, for $\gamma$-reliability to be satisfied, the number of samples $n$ needs to be upper bounded by $\gamma$. It is thus expected that the importance distribution $q$ satisfying
    \begin{equation}
        \bar{r} \cdot \Bigg(e^{-\frac{c}{4}} + 2\sqrt{\mathbb{P}\Big(\log \frac{p(x)}{q(x)} > L + \frac{c}{2}\Big)} \Bigg) \leq \tau' = \frac{\beta\tau}{\beta+1},
    \end{equation}
    and $\gamma \geq e^{L + c} = e^{D(p \mid\mid q) + c} \Rightarrow c \leq -D(p \mid\mid q) + \log\gamma$
    for some $t \in \R_{\geq0}$.
    
    This completes the proof.
\end{proof}

\subsection{Discussions}
Note in Algorithm~\ref{alg:r2_tight}, if $c_r$ $p$, $q$, $\bar{r}$ in line 1 are replaced with the confidence coefficient $\epsilon \in (0,1)$, if $\alpha= \frac{2\tau}{\beta+1-2\epsilon}$ in line 2, and if line 4-7 are replaced with $n = \frac{4 \log{(2/\epsilon)}}{2\tau^2(\beta-2\epsilon)^2}$, the algorithm becomes the exact procedure described in Section~\ref{sec:prob:direct}. These  differences enable the transfer from a distribution agnostic approach to a distribution dependent one, tailing an importance sampling based accelerated risk assessment algorithm with repeatable and reliable efforts.

Specifically, $\bar{r} \in [0,1]$ is the maximum risk estimate among a certain set of possible testing subject systems $\mathcal{F} \subset S^{S \times U \times W}$, such as the worst-case risk estimate for all vehicles joining a testing program. Note that having a general risk estimate testing algorithm that is efficiently reliable for all $f \in S^{S \times U \times W}$ is fundamentally \emph{impossible}, as demonstrated in~\cite{weng2022rethink}. The estimate of $\bar{r}$ is necessary to confine the tests to a subset of $\subset S^{S \times U \times W}$ and is readily obtainable in the testing practice. It may also be conservatively over-estimated to err on the side of caution, albeit at a potential cost to performance, as will be shown later. 

Moreover, the nominal and importance distributions are $p$ and $q$, respectively. Note in the general statistics applications, $p$ and $q$ may not be known properties~\cite{chatterjee2018sample}, hence the presented results may not be as useful or applicable. However, in the context of testing, $p$ is often expected to be provided as part of the specifications made by standard organizations and regulatory authorities, while $q$ is required for testing execution by operators and testing agencies. 

\begin{remark}\label{rmk:not-matter}
Not every importance sampling based testing algorithm is an accelerated testing algorithm. The accelerated performance largely relies on the properly chosen importance distribution $q$, which is out of the scope of this paper~\cite{zhao2016accelerated,zhao2017accelerated,feng2020adaptive,feng2021intelligent,feng2023dense}. The results in this section 
are applicable for an arbitrary pair of target and importance distributions $(p,q)$.  
\end{remark}

The repeatability of Algorithm~\ref{alg:r2_tight} is achieved through a trade-off with accuracy as mentioned in Section~\ref{sec:prob:direct}. Line 9-12 specifies the return as the medium value of the interval with range $\alpha$ that contains the original risk estimate. In practice, for any established $\alpha$ (based on the desired accuracy and probability of repeatability) among stakeholders, $\alpha_0$ can be sampled by the testing initiator (e.g., the regulatory authority or a standard organization) as a ``secret'' code shared among testing operators to ensure the desired $\beta$-repeatable property. 

The reliability of Algorithm~\ref{alg:r2_tight}, revealed by Theorem~\ref{thm:r2_tight}, is jointly determined by the Kullback–Leibler (KL) divergence between the two distributions and the value of $c$ which is mainly associated with the typical order of fluctuations of $\log(p(X)/q(X))$ from its expected value. In practice, the KL-divergence between $p$ and $q$ is immediate given the above, and one can numerically approximate the minimum estimate of $c$ that satisfies~\eqref{eq:r2_tight_ieq} (e.g., an incremental trial-and-error search of $c$ values from zero as demonstrated in line 5-6 in Algorithm~\ref{alg:r2_tight} or some other bisection alternatives). Note $c_r$ is preferably a small value to enhance the efficiency, but the search with large steps will not undermine the theoretical properties dictated from Theorem~\ref{thm:r2_tight}. These two values (order of fluctuations and KL-divergence) critically influence the necessary testing efforts to minimize the risk estimation error (subject to the $\ell_1$-norm bound of $\tau$) with a high probability of $1-2e^{-2n\tau^2}$, thereby insuring $\gamma$-reliability (and $\beta$-repeatability). 

Depending on the nature of $f$ (revealed through $\bar{\tau}$), $p$, and $q$, the desired testing efforts and the level of assurance could vary significantly. However, in contrast to the conventional approach using RHW for test termination, the testing effort for any subject system can be precisely predetermined and adjusted prior to any testing, enabling tests to be executed with reliably repeatable outcomes.

\section{Experiments}\label{sec:exp}
The proposed algorithm is widely applicable across diverse problem sets, given knowledge of $p$, $q$ and $\bar{\tau}$. While it can address the issues depicted in Fig.~\ref{fig:is_example} based on Feng et al.~\cite{feng2021intelligent}, the distribution of $p$ utilized in that study was not explicitly revealed to the public~\cite{nadegithub}. This work considers two examples of applying push-over disturbances on a controlled inverted pendulum and a 7-DoF planar bipedal robot, Rabbit, as illustrated in Fig.~\ref{fig:demos}. The examples provided exhibit dynamic and stochastic responses similar to those of various humanoid and legged robot systems. This similarity presents numerous challenges for standard testing algorithms, as discussed in the existing literature~\cite{weng2022safety,weng2022rethink,weng2023towards}. One can also refer to the open-source code\footnote{\href{https://tinyurl.com/rp-rl-risk-est}{https://tinyurl.com/rp-rl-risk-est}} for a synthesized non-robotics related example.

\begin{remark}
    The primary focus of this work is on testing algorithms, specifically their repeatability and reliability. The subjects being tested, while necessary, are not the central interest. The controllers for these subjects are adaptations from the works in \cite{varghese2017optimal,gong2021one,castillo2023template}; and therefore, detailed descriptions are omitted.
\end{remark}

\begin{figure}
    \centering
    \includegraphics[trim={2cm 0cm 3cm 0}, clip, width=0.35\textwidth,angle=270]{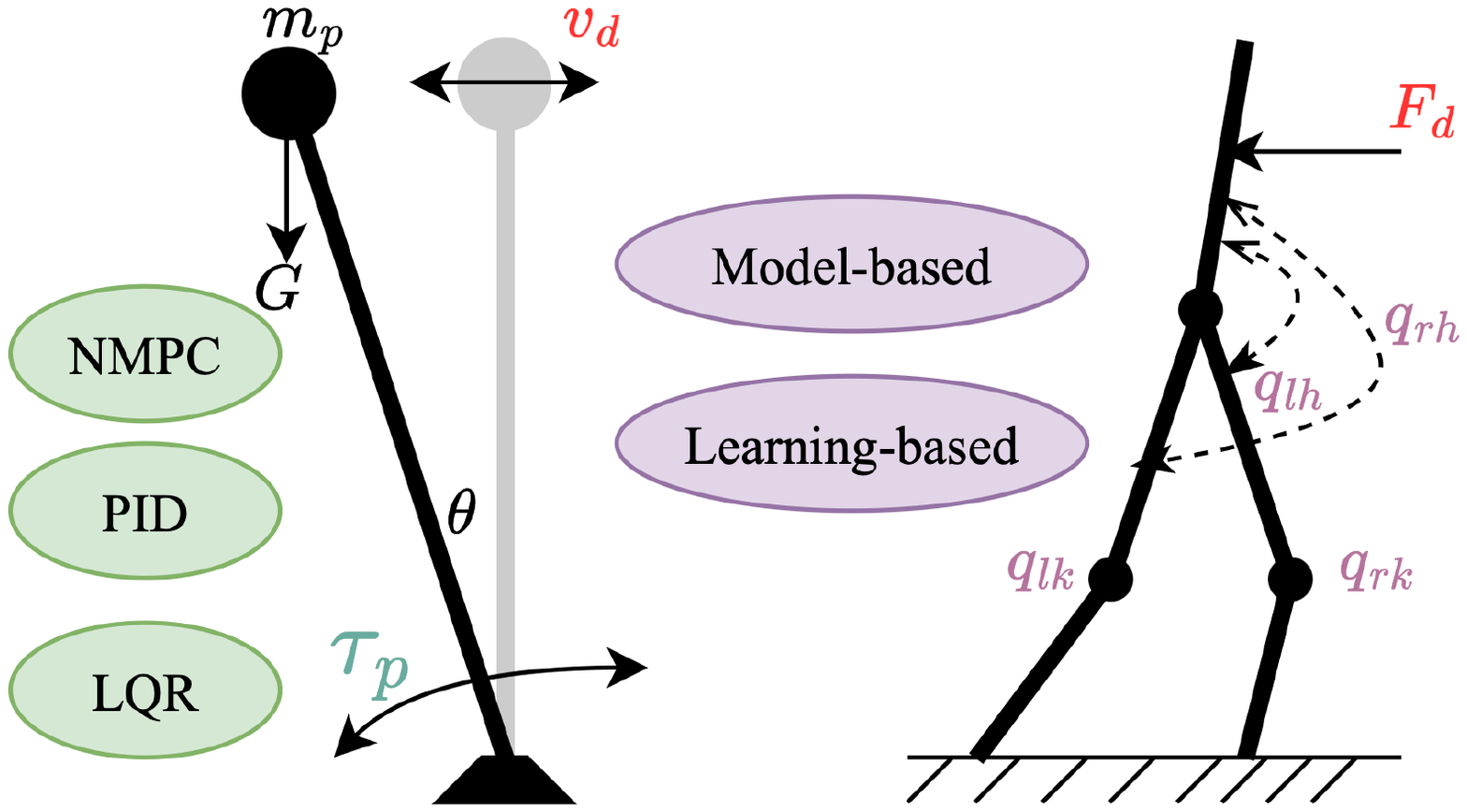}
    \caption{The inverted pendulum push-over task (left) and the legged robot pushover task with Rabbit~\cite{castillo2019reinforcement,gong2021one} (right). Three controllers are tested for the inverted pendulum push-over task. Two locomotion controllers are evaluated for the Rabbit case.}
    \label{fig:demos}
    \vspace{-5mm}
\end{figure}

\begin{figure*}[t]
    \centering
    \includegraphics[trim={1.5cm 6.5cm 1.5cm 0.8cm}, clip, width=0.99\textwidth]{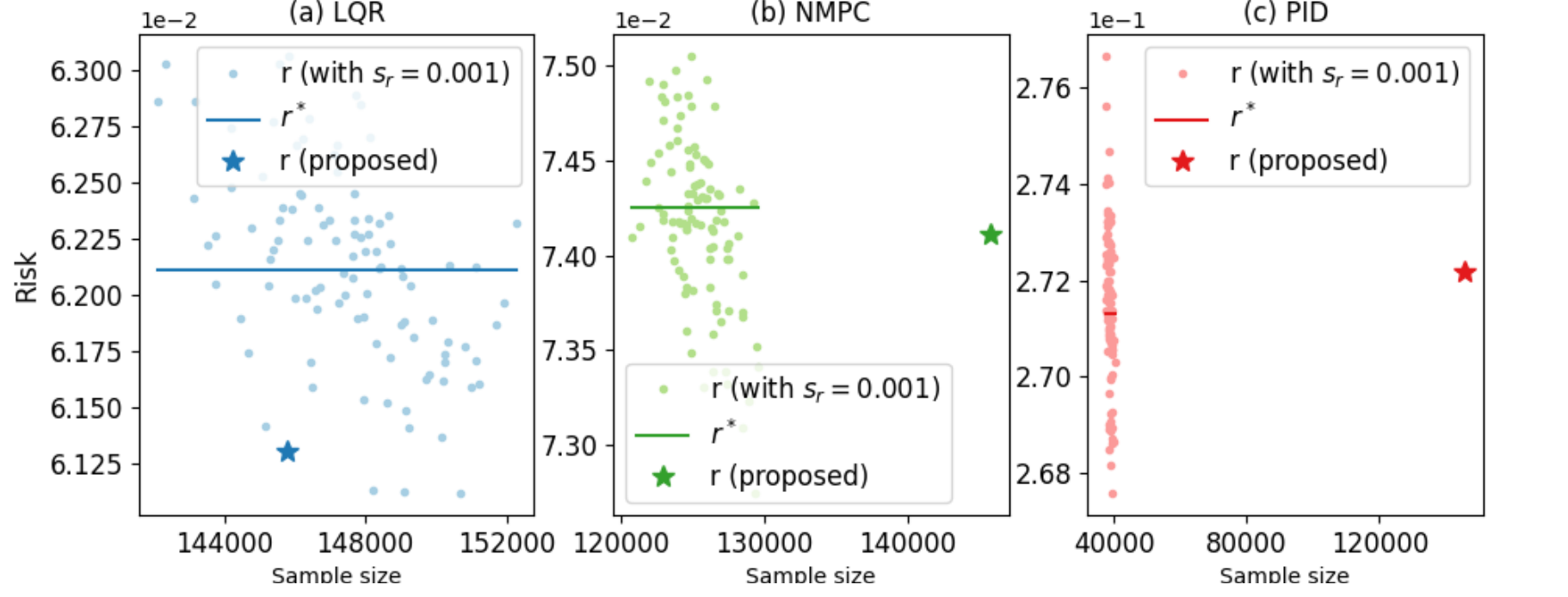}
    \vspace{-2mm}
    \caption{Repeatability and reliability comparison between Algorithm~\ref{alg:r2_tight} and Algorithm~\ref{alg:is_testing} with RHW-based termination condition through the inverted-pendulum push-over task against 3 different controllers (LQR, NMPC, and PID). Within each sub-figure, Algorithm~\ref{alg:r2_tight} is configured using identical specifications with $\beta=0.4, \bar{t}=0.3, \tau=0.1$, which leads to 145800 samples and a failure probability ($\epsilon$) that is extremely close to zero. It is repeated for 100 times against each controller. The obtained risk estimates are identical shown as the star in each sub-figure. The dots in all figures represent another 100 attempts of Algorithm~\ref{alg:is_testing} with the RHW threshold of $s_r=0.001$ as the termination criterion. The $r^*$ for each controller is obtained through the approximated enumeration of all samples in the discretized sample space $V_d$ with sufficiently small resolution ($0.002$ m/s).}
    \label{fig:utilit_pendulum}
    \vspace{-5mm}
\end{figure*}

\begin{figure}[h]
    \centering
    \includegraphics[trim={4cm 1 0cm 0}, clip,width=0.95\textwidth]{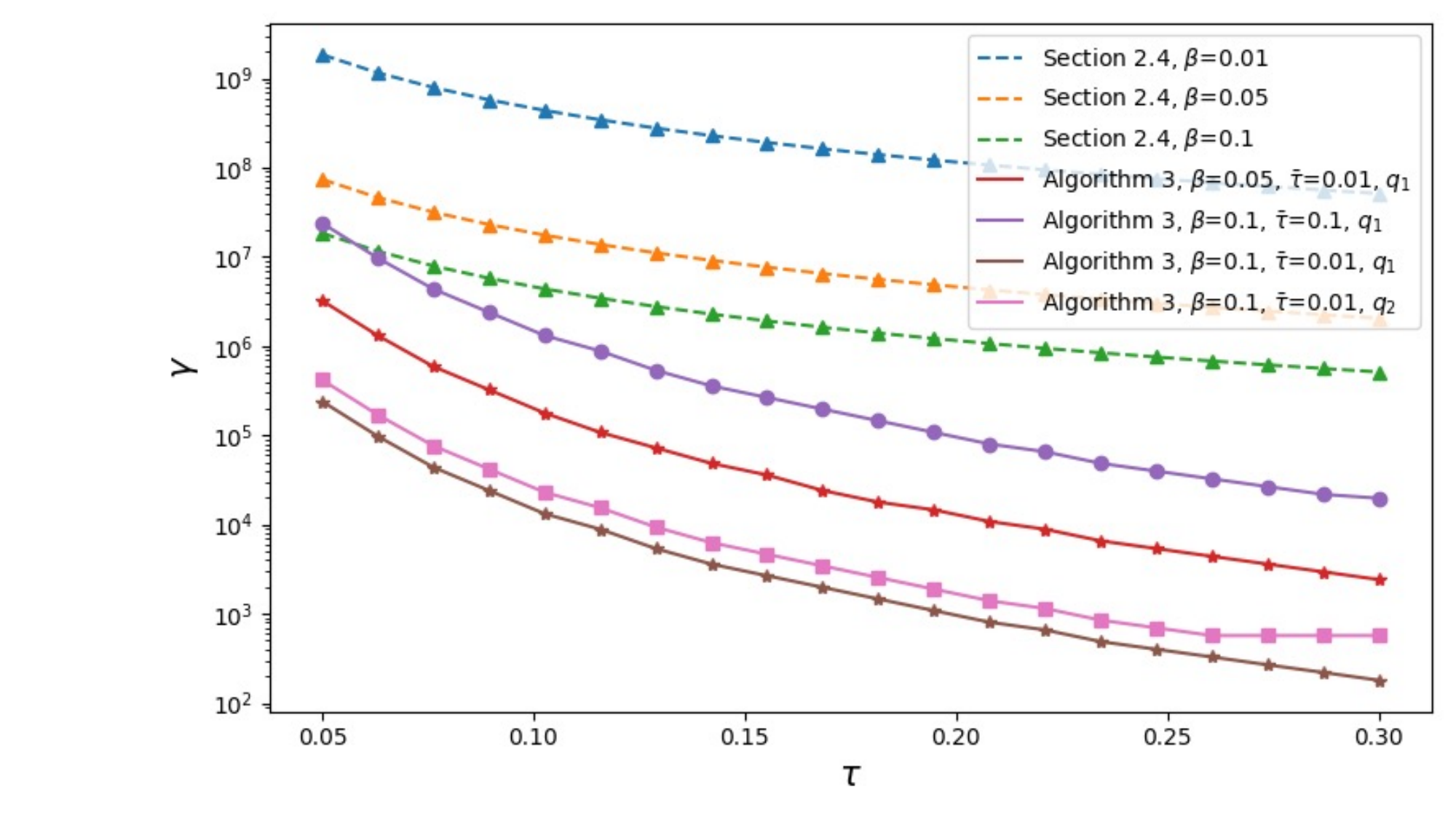}
    \caption{Comparing the main proposal (Algorithm~\ref{alg:r2_tight}) and the described procedure in Section~\ref{sec:prob:direct} w.r.t. a variety of hyper-parameters. Note the $\gamma$ values are shown in the $\log$ scale.}
    \vspace{-5 mm}
    \label{fig:effort_pendulum}
\end{figure}

\subsection{The inverted pendulum pushover} 
\textbf{Testing Configuration}: This standard case in the dynamics and control literature involves a vertically placed \emph{pendulum} (a one-meter-long rod with a 1 kg mass at one end and an actuated pivot at the other). A certain mechanism applies torque ($\tau_p \in [-1,1]$ N$\cdot$m) to the pivot point to control the pendulum's angle (subject to constraints $|\dot{\tau}_p| \leq 10$ N$\cdot$m/s). The coefficient of friction is $0.5$ and the gravity is $9.817$ m/s$^2$. Three different controllers are implemented: (i) a Proportional-Integral-Derivative (PID), (ii) a Linear Quadratic Regulator (LQR), and (iii) a nonlinear Model Predictive Controller (NMPC)~\cite{varghese2017optimal}. 
Each controller is able to maintain the pendulum in an upright position at equilibrium ($\theta=\dot{\theta}=0$) within certain disturbance levels of the initial velocity $v_d$. However, due to the system's nonlinear nature and the specific constraints on $\tau_p$, none of the controllers can achieve a ``swing-up'' from a fully inverted position to equilibrium. As a result, with a sufficiently large $v_d$ (0.9 m/s), all controllers will fail. However, the respective risk of each controller is unclear. Given a nominal distribution $p$ for $v_d$ as a truncated normal distribution (of zero mean and a standard deviation of $0.5$), large magnitudes of $v_d$ that could fail the controller are thus clustered in the distribution ``tails''. Two importance functions $q_1$ and $q_2$, are introduced, each biased towards large magnitudes of $v_d$ in distinct ways within the sample space of $V_d = [-0.9,0.9]$. 

This particular case is chosen for two reasons. First, its dynamics are simple enough to obtain a nearly perfect ``ground truth'' risk estimate. Second, and more importantly, it theoretically emulates more challenging problems in the robot testing practice such as the push-over disturbance rejection tests of legged robots~\cite{weng2022safety,weng2023towards}. The intelligent feedback of a legged robot could amplify the uncertainties and the complexity of $f$, as shown by previous works in simulation and real-world~\cite{weng2022safety,weng2022rethink}. However, the proposed algorithms are not relying on those insights (i.e., the knowledge of uncertainties and $f$) anyway. The insights garnered from this case study not only shed light on the observed phenomenon but also hint at the potential applicability of these algorithms to the testing of legged robots, as presented with the Rabbit case in Section~\ref{sec:rabbit}.

\noindent\textbf{Observations and Analyses}: 
Fig.~\ref{fig:effort_pendulum} demonstrates the different outcomes of $\gamma$ (required number of sampled tests) and $\tau$ (estimation error bound in $\ell_1$-norm) presented with the proposed testing algorithm and the direct extension of~\cite{impagliazzo2022reproducibility} described in Section~\ref{sec:prob:direct}. Algorithm~\ref{alg:r2_tight} generally outperforms the other in terms of testing efficiency with similar estimation performance. As the $\bar{\tau}$ becomes large and $\tau$ becomes small, the difference between the two algorithms are less pronounced, which also aligns with the theoretical analysis. In practice, the $\gamma$-reliability can also be justified taking the smaller value between the two $\gamma$s revealed through Theorem~\ref{thm:r2_tight}. 

When comparing the proposed Algorithm~\ref{alg:r2_tight} with the traditional practice of RHW-based termination of importance sampling, the results are shown in Fig.~\ref{fig:utilit_pendulum}. Given the hyper-parameters set for Algorithm~\ref{alg:r2_tight}, Theorem~\ref{thm:r2_tight} guarantees a probability of 0.6 that the algorithm is repeatable, and the estimation error is bounded by $0.1$ with a close-to-one probability. Empirically revealed by the particular tests in Fig.~\ref{fig:utilit_pendulum}, Algorithm~\ref{alg:r2_tight} is shown 100\% repeatable, 100\% accurate within the error bound $\tau$, and reliable across the tests among all three subject controllers with the same testing effort. In comparison, when $s_r=0.001$ is pre-determined as the threshold for Algorithm~\ref{alg:is_testing} taking the RHW based termination condition, none of the tests is repeatable (revealed by the spreading dots on all sub-figures) or reliable (revealed by different sample sizes). Also, it is noted that for the non-repeatable outcomes of Algorithm \ref{alg:is_testing}, the estimation error does not decrease w.r.t. the sample size, which (i) indicates the $s_r$ threshold is already sufficiently small, and (ii) further addresses the non-repeatable and non-reliable nature of the commonly adopted RHW-based termination approach. Moreover, it may be noted that the average risk estimate across 100 attempts of Algorithm~\ref{alg:is_testing} appears to approximate the true risk ($r^*$) of the subject controllers under examination. Nonetheless, it's important to underline that the algorithm, as per its design, is not meant to be executed across such an extensive series of trials; standard practice dictates a \emph{single} execution. Implementing a procedure that necessitates multiple attempts and then determines the final risk estimate as the mean of these attempts would invariably escalate the testing effort.

\subsection{The legged-robot pushover}\label{sec:rabbit}
\noindent\textbf{Testing configuration}: The subject being tested is a 7-DoF 5-link planar legged robot named Rabbit, which is a widely studied benchmark legged robot system~\cite{castillo2019reinforcement,chevallereau2003rabbit}. Two fundamentally different locomotion controllers are implemented, including a model-based one using the Angular Momentum Linear Inverted Pendulum (ALIP) method~\cite{gong2021one}, and another incorporating template-based Reinforcement Learning (RL)~\cite{castillo2023template}. Both controllers are capable of walking forward and backward at a variety of speeds on flat surface without the presence of other disturbances. The testing action is the frontal push-over force $F_d$ following a parameterized periodic wave function $F_d(t) = M_d \text{sgn}(\sin(2\pi f_d t))$. We further have the magnitude $M_d \in [-20, 80]$ N and frequency $f_d \in [0.1,10]$ Hz. During the initialization of each test, the robot is commended to reach a steady-state walking speed of $v_s \in V_s = [-1,1]$ m/s, followed by the frontal impact disturbance determined by a sample $(M_d, f_d)$ over the parameter sample space $U_d = [-20, 80] \times [0.1,10]$ following a probability distribution. Since there are no legged robot testing standards available in providing the configuration of the appropriate target distribution $p$ or existing literature exploring the development of the importance distribution $q$ for this type of application~\cite{weng2022safety}, the following experiment assumes $q$ is a uniform distribution over $V_s \times U_d$, and $p$ is a truncated multi-variable normal distribution over the same sample space. Note with the intelligent and dynamic feedback of legged robot locomotion controllers, there does not exist any simple monotonic relation between the safety performance (risk of falling-over) and the disturbance parameters. For example, a large $M_d$ does not necessarily lead to a high risk. This has been revealed extensively with legged robot locomotion testing in simulation~\cite{weng2022safety} and in real-world~\cite{weng2023towards}. As a result, designing the appropriate importance distribution $q$ is also a challenge and is out of the scope of this paper, yet it does not affect the main contribution of this paper as discussed in Remark~\ref{rmk:not-matter}

\begin{figure*}[t]
    \centering
    \includegraphics[trim={2cm 6cm 1.5cm 0.8cm}, clip, width=0.99\textwidth]{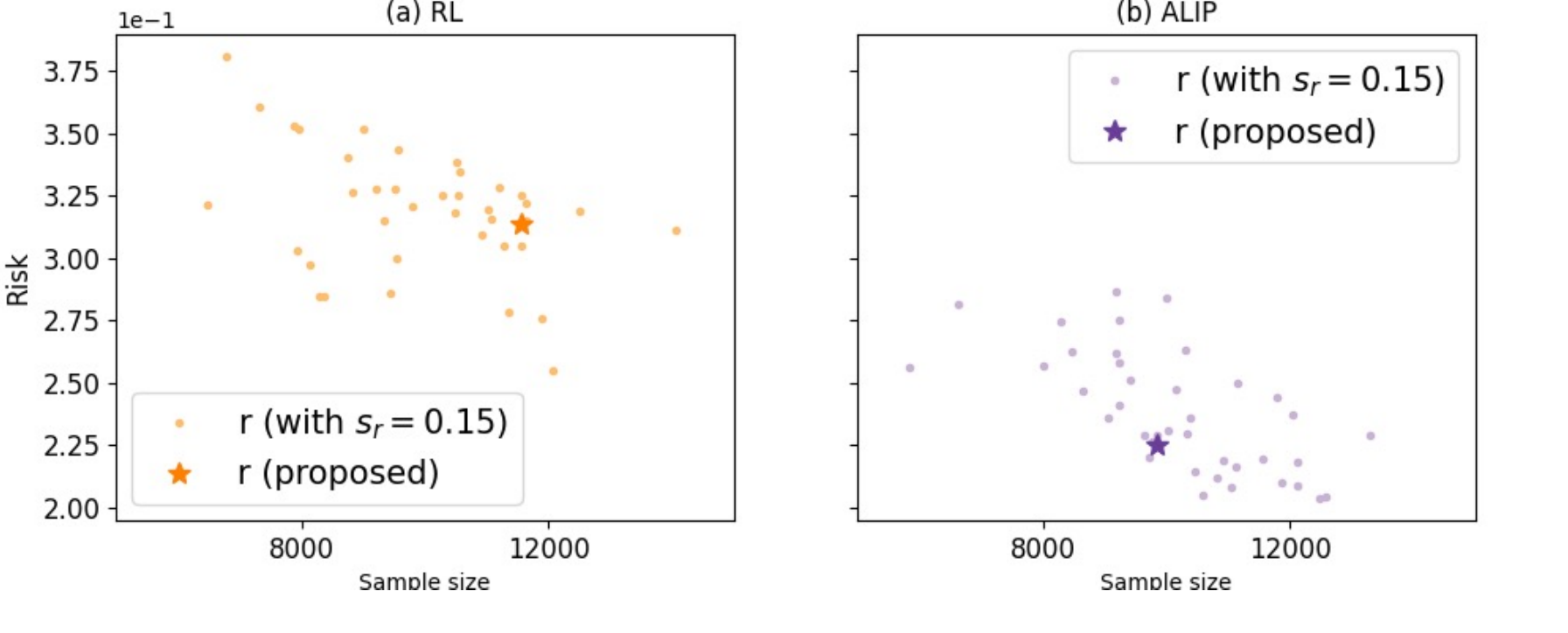}
    \vspace{-5mm}
    \caption{Repeatability and reliability comparison between Algorithm~\ref{alg:r2_tight} and Algorithm~\ref{alg:is_testing} with RHW-based termination condition for the Rabbit push-over task against two different controllers (RL and ALIP). Within each sub-figure, Algorithm~\ref{alg:r2_tight} is configured using identical specifications with $\beta=0.4, \bar{t}=0.35, \tau=0.2$, which leads to 10939 samples and a failure probability ($\epsilon$) that is extremely close to zero. It is repeated 35 times against each controller. The obtained risk estimates are identical, shown as the star in each sub-figure. The dots in all figures represent another 35 attempts of Algorithm~\ref{alg:is_testing} with the RHW threshold of $s_r=0.13$ as the termination criterion. Note $r^*$ is not presented in this study as there is no trackable solution to obtain the ground-truth risk estimate for this Rabbit push-over case.}
    \label{fig:utilit_rabbit}
    \vspace{-3mm}
\end{figure*}

\noindent\textbf{Observations and analyses}: From Fig.~\ref{fig:utilit_rabbit}, the repeatability and reliability properties are immediate. The RL policy also presents a higher risk than the ALIP policy under the same set of coefficients of $\beta$, $\gamma$, and $\alpha$. However, among the 1225 pairs ($35 \times 35$) of testing attempts of Algorithm~\ref{alg:is_testing} with RHW-based stopping criteria against RL and ALIP, 35 of them end up considering RL being of a lower risk. This reveals a particular problem of non-repeatable tests when subject performance are sufficiently close, leading to the wrong justification of the safer subject.

\section{Conclusion} \label{sec:conclusion}
To the best of our knowledge, this work has presented the first formal study of the importance sampling-based accelerated risk assessment testing algorithm that repeatably outputs the same risk estimate, and reliably works across a variety of testing subjects with high probability. Extending these findings to other robots, such as 3D legged robots, appears promising, though the challenge is not primarily related to the robot's configuration, as the proposal is agnostic of $f$ and is only sensitive to the distributions $p$ and $q$. A potential direction for further research could be to refine the sample size bounds presented in Theorem~\ref{thm:r2_tight}, which could significantly enhance testing efficiency. Another direction is to explore its applicability for different types of performance measures beyond risk, including Boolean outcomes and complex multi-dimensional measures.


\begin{credits}
\subsubsection{\disclaimer} Positions and opinions presented in this work are those of the authors and not necessarily those of Iowa State University, The Ohio State University, and Transportation Research Center Inc. Responsibility for the content of the work lies solely with the authors.


\end{credits}
%
%
%
\bibliographystyle{splncs04}
\bibliography{output}
%




\end{document}